\documentclass[letterpaper,10pt,conference]{ieeeconf}

\IEEEoverridecommandlockouts
\usepackage{amsmath}
\usepackage{amssymb}
\usepackage{booktabs}
\usepackage{cite}
\usepackage{graphicx}
\usepackage{placeins}
\usepackage{url}
\usepackage{hyperref}

\newtheorem{proposition}{Proposition}

\title{\LARGE \bf
BladeMaster: Real-Time Robotic Cutting Simulation with Online-Generated Persistent Discontinuities
}

\author{
Zhanyu Yang$^{1,2}$,
Yunuo Chen$^{2,3}$,
Yanjia Huang$^{3}$,
Joseph Masterjohn$^{4}$,
Yin Yang$^{5}$,
Chenfanfu Jiang$^{2}$
\thanks{$^{1}$Purdue University, USA.}%
\thanks{$^{2}$UCLA, USA.}%
\thanks{$^{3}$Envora, USA.}%
\thanks{$^{4}$Toyota Research Institute, USA.}%
\thanks{$^{5}$University of Utah, USA.}%
}

\begin{document}

\bstctlcite{IEEEtranBSTcontrol}

\maketitle
\thispagestyle{empty}
\pagestyle{empty}

\begin{abstract}

Cutting changes both the shape and topology of deformable objects, making accurate simulation challenging for robotic manipulation. A simulator must track the cutting tool as a cut develops, preserve the resulting discontinuities after tool withdrawal, and enable newly exposed surfaces to interact with the tool and with each other. Existing formulations often prescribe cut surfaces in advance or couple material separation to auxiliary geometric fields. We introduce \textsc{BladeMaster}, a GPU-accelerated cutting framework based on the total Lagrangian material point method (TLMPM). Our key idea is to encode the cutting history directly on material points through persistent side labels generated online from the blade geometry. These labels govern particle--grid coupling, preserving connectivity within intact material while preventing spurious coupling across cut faces after tool withdrawal. Our formulation supports progressive and intersecting cuts without predefined cut surfaces or particle duplication. Material--material contact enables cut surfaces to recontact and slide against each other without reconnecting, while two-way tool--material coupling allows material reaction forces to influence tool motion. Experiments demonstrate tool-driven cutting followed by manipulation, with faster-than-real-time performance on representative tasks. Project page: \href{https://jango6324.github.io/blademaster/}{https://jango6324.github.io/blademaster/}.

\end{abstract}

\section{Introduction}
\label{sec:introduction}

Cutting is an important manipulation capability for robots working with soft materials. As the blade advances through an object, it deforms the material, separates neighboring material regions, and interacts with newly exposed surfaces. Real-world cutting experiments are costly because each cut irreversibly damages the sample. An accurate and efficient simulator can therefore support material calibration, motion planning, and policy learning while reducing the need for physical experiments~\cite{Heiden2021DiSECt,Xu2023RoboNinja,Wang2025TopoCut}. Such a simulator must generate cuts online as the cutting tool moves and preserve the resulting discontinuities during subsequent contact and manipulation.

\begin{figure}[!t]
  \centering
  \includegraphics[width=\columnwidth]{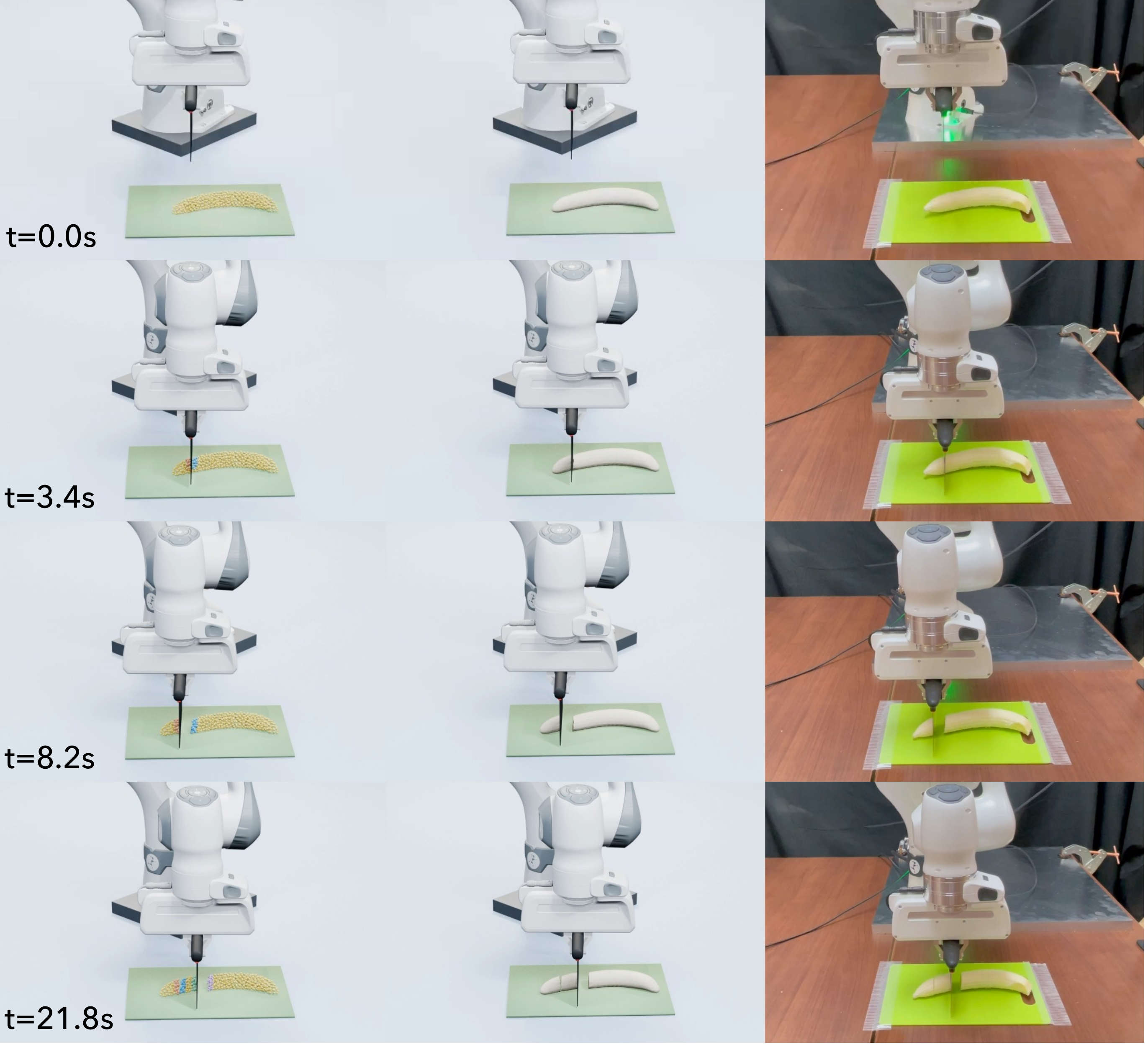}
  \caption{Sequential banana cutting and manipulation in simulation and in a real-world experiment. After each cut, the robot uses the cutting tool to push the resulting pieces. The columns show simulated particles (left), reconstructed surfaces (middle), and the physical experiment (right). From top to bottom, the rows show the states before cutting, during the first cut, after the first cut and push, and after the second cut and push.}
  \label{fig:banana-teaser}
  \vspace{-0.5cm}
\end{figure}

Existing approaches differ in how they create cuts and preserve the resulting separation. The mesh-based approach DiSECt~\cite{Heiden2021DiSECt} requires the geometry of the entire cut surface to be specified before simulation and constructs virtual nodes and cutting springs along it during preprocessing. Material separation is therefore confined to this predefined surface. In contrast, the material point method (MPM) is well suited to large-deformation simulation and can accommodate complex, evolving damage, but its shared grid states can introduce artificial coupling across a cut. CPIC~\cite{Hu2018MLSMPM} addresses this issue using a colored distance field to block particle--grid transfers across the cutting-tool surface. This separation, however, remains tied to the tool geometry, so nearby cut faces can become artificially coupled again after tool withdrawal. CRESSim-MPM~\cite{Ou2025CRESSim} further simplifies the enforcement of separation across cutting boundaries using signed distance functions. However, its formulation does not retain the cumulative changes in material connectivity induced by successive cuts, limiting the independent manipulation of the resulting pieces.

We introduce \textsc{BladeMaster}, a GPU-accelerated framework that generates cuts online from blade motion and retains them as part of the material state. We build on the total Lagrangian material point method (TLMPM), which uses a compact grid defined in the object's reference configuration, keeping potential coupling neighborhoods fixed regardless of material motion~\cite{deVaucorbeil2020TLMPM}. TLMPM eliminates cell-crossing errors and numerical fracture but requires material separation to be represented explicitly. We encode this separation by augmenting material points with persistent side labels assigned using the swept blade geometry. These labels determine which material points can exchange momentum through the fixed reference grid. This preserves coupling within intact material while maintaining mechanical separation between opposing cut faces after cutting-tool withdrawal. Successive cuts can be introduced through new side labels without overwriting earlier assignments, thereby preserving accumulated material discontinuities and enabling independent manipulation of the resulting pieces.
The persistent side labels also allow newly created cut surfaces to continue interacting with the tool and with each other. Our contact model handles these interactions while preserving the distinction between opposing cut faces, allowing them to recontact and undergo relative sliding without reconnecting. Two-way coupling between MPM and a rigid-body solver transfers material reaction forces to the tool, allowing its motion to respond to the evolving material state.

We evaluate the framework in cutting and manipulation scenarios and compare it against prior methods. We validate the simulated cutting and manipulation behavior against a physical robot experiment (Fig.~\ref{fig:banana-teaser}) and evaluate the predicted cutting resistance against physical measurements. Our GPU implementation simulates cutting and subsequent manipulation at faster-than-real-time rates on representative tasks. Our main contributions are as follows:

\begin{itemize}
  \item A TLMPM-based cutting formulation that assigns persistent side labels to material points online to govern momentum-conserving particle--grid coupling and preserve discontinuities after tool withdrawal. The formulation supports progressive and intersecting cuts without predefined cut surfaces or particle duplication.

  \item A real-time GPU framework that integrates evolving cut topology with material--material contact and two-way tool--material coupling, enabling cut surfaces to recontact and slide against each other without reconnecting and supporting subsequent manipulation of the separated pieces.
\end{itemize}

\section{Related Work}
\label{sec:related-work}

\paragraph{Cutting and discontinuity-aware simulation}
Mesh-based cutting methods explicitly modify or enrich the simulation topology. Virtual-node methods duplicate portions of intersected elements~\cite{molino2004virtual,sifakis2007arbitrary}, while XFEM represents discontinuities through enrichment without conforming remeshing~\cite{jerabkova2009stable,koschier2017robust}. Generalized XFEM extends this capability to complex and intersecting cuts through robust Boolean operations~\cite{ton2024generalized}. These methods provide explicit representations of cut surfaces but require cut-dependent enrichment, integration, or topology bookkeeping. For robotic cutting, DiSECt combines FEM, virtual nodes, cutting springs, and differentiable simulation for material calibration and trajectory optimization, while requiring the entire cut surface to be specified during preprocessing~\cite{Heiden2021DiSECt}. Wind Lifter incorporates winding-number coordinates into a neural reduced-order model for real-time cutting of thin structures~\cite{chang2025lifting}.

\paragraph{Fracture and cutting with MPM}
MPM accommodates large deformation and material separation without conforming remeshing. Early work introduced multiple velocity fields for explicit cracks~\cite{nairn2003cracks} and cohesive zones for tool-driven cutting~\cite{nairn2015cutting}. Subsequent damage-based approaches model continuum-damage and phase-field fracture~\cite{wolper2019cdmpm}, nonlocal isotropic and anisotropic damage~\cite{wolper2020anisompm}, or fragmentation with evolving multi-body contact~\cite{xiao2021dpmpm}. More recent work uses damage-informed enriched basis functions~\cite{sugai2024diffusive} and models damage and tearing within an embodied differentiable framework calibrated from real-world observations~\cite{chen2026empm}. These damage-based methods primarily address material-driven fracture, whereas robotic cutting requires separation to evolve with the moving tool. CPIC~\cite{Hu2018MLSMPM} combines a colored distance field with particle--grid compatibility to block transfers across a cutting surface while supporting two-way coupling. However, particle--grid transfers across the cut may resume after tool withdrawal. CRESSim-MPM~\cite{Ou2025CRESSim} uses a unified SDF representation for both cut-boundary separation and rigid contact. SDF signs determine particle--grid compatibility, while absolute distances combined with a fattened collision region are used for contact resolution. However, the method does not encode cumulative changes in material connectivity from successive cuts as part of the material state. TopoCut~\cite{Wang2025TopoCut} combines MLS-MPM, progressive damage, and topology reconstruction for multi-step robotic cutting and policy learning. However, its reported particle--grid transfers remain single-field and do not distinguish fragment identities.

\paragraph{Total-Lagrangian MPM and contact modeling.}
TLMPM evaluates particle--grid interactions in the reference configuration, maintaining fixed neighborhoods and avoiding cell-crossing errors and numerical fracture~\cite{deVaucorbeil2020TLMPM}, while contact between bodies with independent reference grids must be handled explicitly. Particle-based~\cite{deVaucorbeil2021Contact} and contact-grid~\cite{bui2025tlmpmcontact} formulations address these interactions, with the latter offering improved scalability. These methods handle contact between already distinct bodies, whereas cutting also changes material connectivity as the blade advances. Within TLMPM, \textsc{BladeMaster} uses persistent cut-side labels to control reference-grid coupling, while explicit contact preserves the separation of opposing cut surfaces during recontact and relative sliding.

\section{Method}
\label{sec:method}

\subsection{System and Notation}
\label{sec:method-overview}

We simulate a deformable body in contact with a moving rigid cutting tool. The body is discretized into material points and advanced in time using TLMPM\@. At time step $n$, the tool pose $(\mathbf{x}_B^n,\mathbf{R}_B^n)$ and twist $(\mathbf{v}_B^n,\boldsymbol{\omega}_B^n)$ define the contact geometry and motion over an MPM substep. The contact solver applies the resulting impulse $\mathbf{J}_{pB}$ to particle $p$ and returns the corresponding reaction wrench on the tool (Section~\ref{sec:method-contact}). An external rigid-body system handles tool actuation, constraints, and state updates.

Particle $p$ stores its reference position $\mathbf{X}_p$, current position $\mathbf{x}_p^n$, velocity $\mathbf{v}_p^n$, deformation gradient $\mathbf{F}_p^n$, reference volume $V_p^0$, and mass $m_p$. We augment this standard TLMPM state with a cutting-progress variable $D_p\in[0,1]$ and a cut-side code $c_p$. At the beginning of each new cut, we reset $D_p$ to zero. During each substep, material--material and tool--material contact impulses are scattered to the TLMPM grid. Valid tool contact advances $D_p$ and modulates the resistance opposing the blade, while the swept blade band updates $c_p$ according to the particle's side relative to the blade. Starting with the next TLMPM substep, the resulting codes determine which particle contributions may share grid states (Fig.~\ref{fig:method-grid-coupling}). The same blade sweep disconnects intersected links in the reference grid and triggers local surface reconstruction.

\begin{figure}[t]
  \centering
  \includegraphics[width=0.95\columnwidth]{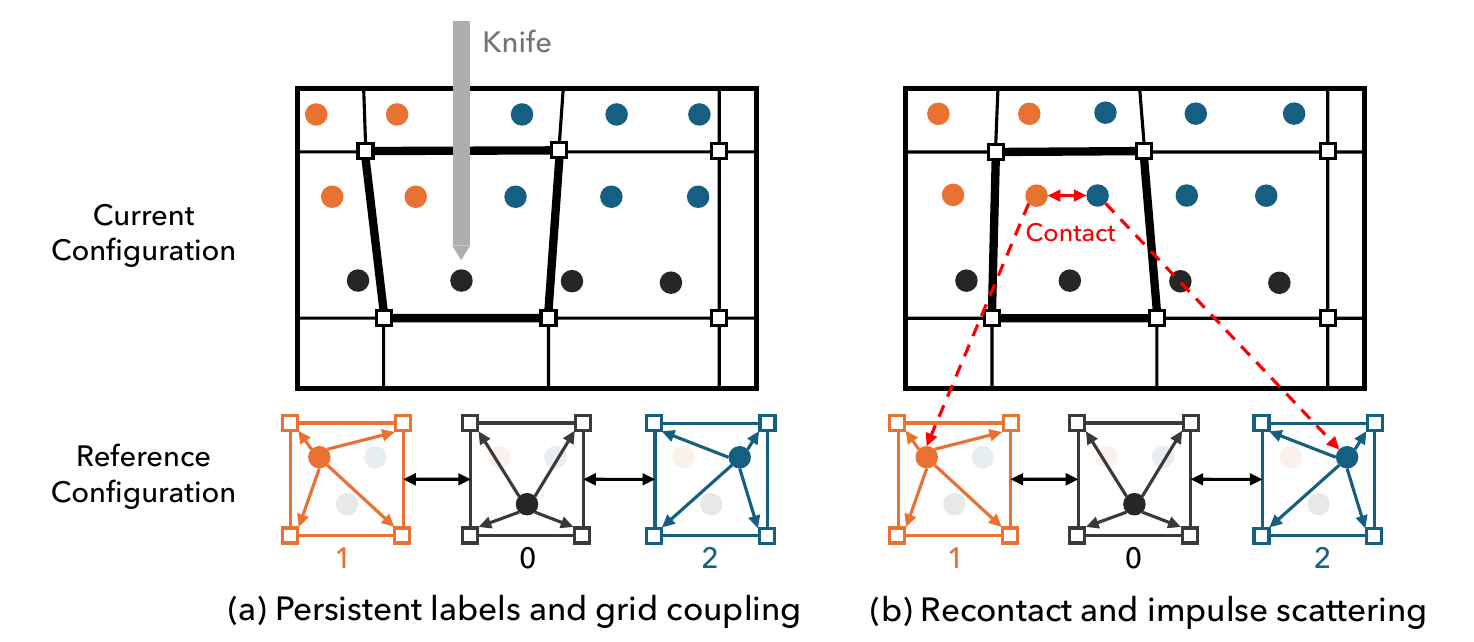}
  \vspace{-3mm}
  \caption{Persistent cut-side labels and grid coupling. (a) Labels 1 and 2 denote opposite sides of a cut, while label 0 denotes unassigned material. Grids 1 and 2 each couple with grid 0 but do not couple directly with each other. (b) After knife withdrawal, the labels persist, and impulses generated upon recontact are scattered to the corresponding grids without restoring direct coupling. The bold-outlined grids in the current configuration and the grids in the reference configuration represent the same cell.
  }
  \label{fig:method-grid-coupling}
  \vspace{-0.5cm}
\end{figure}

\subsection{TLMPM Discretization}
\label{sec:method-tlmpm}

We adopt an explicit modified-update-stress-last (MUSL) scheme for TLMPM~\cite{deVaucorbeil2020TLMPM} and summarize below the quantities required for cutting and contact. Let $i$ index the reference-grid nodes, and let $N_i$ denote the associated interpolation function. Unlike conventional MPM, TLMPM evaluates interpolation weights and their gradients at fixed particle reference positions rather than at their current positions:
\begin{equation}
  N_{ip}=N_i(\mathbf{X}_p),
  \qquad
  \boldsymbol{\nabla}_0N_{ip}
  =\left.\frac{\partial N_i}{\partial\mathbf{X}}
  \right|_{\mathbf{X}_p}.
  \label{eq:method-reference-shape}
\end{equation}
The resulting particle--grid stencil remains constant and can therefore be precomputed. In TLMPM, material motion cannot introduce new grid-mediated coupling outside the original reference neighborhoods. Our formulation can thus control changes in material connectivity solely by updating the cut-side codes. The cut-side code $c_p$ determines the logical grid to which particle $p$ scatters its mass, momentum, and internal-force impulse. For an MPM substep of size $\Delta t$, the reference-configuration P2G transfer independently accumulates these quantities on each logical grid $c$ before coupling compatibility is enforced:
\begin{align}
  m_i^c
  &=\sum_{p:c_p=c}N_{ip}m_p,\nonumber\\
  \mathbf{p}_i^{n,c}
  &=\sum_{p:c_p=c}N_{ip}m_p\mathbf{v}_p^n,\nonumber\\
  \Delta\mathbf{p}_i^{\mathrm{int},c}
  &=-\Delta t\sum_{p:c_p=c}
  V_p^0\mathbf{P}_p^n\boldsymbol{\nabla}_0N_{ip},
  \label{eq:method-reference-p2g}
\end{align}
where $\mathbf{P}_p^n$ is the first Piola--Kirchhoff stress. We use a fixed-corotated constitutive model, optionally augmented by a von Mises-type plasticity projection and Newtonian viscosity. These constitutive models are independent of the cutting formulation. External and contact impulses are accumulated on the same logical grids before the grid update. After enforcing compatibility as described in Section~\ref{sec:method-topology}, we perform the standard explicit grid update, PIC/FLIP G2P transfer, and MUSL deformation-gradient update~\cite{deVaucorbeil2020TLMPM}. Compatibility is enforced during both the initial P2G transfer and the MUSL velocity remapping.

\subsection{Cutting Resistance and Side Assignment}
\label{sec:method-cutting}

For sequential cutting, an empirical law based on directional contact work updates $D_p$ to modulate resistance to the blade, while the swept blade geometry determines persistent side labels that encode topological separation.

\paragraph{Blade local coordinates}
For cut $k$, let $\mathbf{x}_e$ denote a point on the blade edge, and let the world-frame unit vectors $\mathbf{n}_k$, $\mathbf{e}_k$, and $\mathbf{d}_k$ define the cutting-plane normal, the axis along the finite blade edge, and the front--back direction, respectively, with $\mathbf{d}_k$ pointing toward the cutting front. These quantities are obtained by transforming their fixed tool-frame counterparts using the current tool pose. For a particle $\mathbf{x}_p$, we define $\mathbf{r}_{p,k}=\mathbf{x}_p-\mathbf{x}_e$ and compute
\begin{equation}
  s_{p,k}=\mathbf{r}_{p,k}\cdot\mathbf{n}_k,
  \quad
  \ell_{p,k}=\mathbf{r}_{p,k}\cdot\mathbf{e}_k,
  \quad
  f_{p,k}=\mathbf{r}_{p,k}\cdot\mathbf{d}_k.
  \label{eq:method-blade-coordinates}
\end{equation}
Here, $s_{p,k}$ identifies the side of the cutting plane, $\ell_{p,k}$ measures position along the finite blade edge, and $f_{p,k}$ distinguishes material in front of and behind the edge. Let $\mathbf{v}_e$ denote the velocity of the blade point $\mathbf{x}_e$. The relative cut-advance speed along $\mathbf{d}_k$ is
$u_{c,p}=(\mathbf{v}_e-\mathbf{v}_p)\cdot\mathbf{d}_k$.
Given the band half-width $h_k$, edge half-length $L_k$, rear extent $b_{\mathrm{back}}$, and lead extent $b_{\mathrm{lead}}$, we define the finite \emph{cutting band} by
\begin{align}
  |s_{p,k}|&\le h_k,
  \qquad
  |\ell_{p,k}|\le L_k,
  \nonumber\\
  -b_{\mathrm{back}}-\max(u_{c,p}\Delta t,0)
  &\le f_{p,k}\le b_{\mathrm{lead}}.
  \label{eq:method-cut-band}
\end{align}

\paragraph{Directional contact work}
We accumulate the blade-to-particle contact work along the cut-advance direction only for contacts within the front portion of the cutting band ($0 \le f_{p,k} \le b_{\mathrm{lead}}$) and when $u_{c,p}>u_{\min}$, where $u_{\min}$ is the minimum cut-advance speed:
\begin{align}
  J_{c,p}
  &=\max(\mathbf{J}_{pB}\cdot\mathbf{d}_k,0),
  \quad \Delta W_{p,k}=J_{c,p}u_{c,p},\nonumber\\
  D_p&\leftarrow\min\left(
  D_p+\frac{\Delta W_{p,k}}{W_{\mathrm{th},p}},1\right).
  \label{eq:method-contact-work}
\end{align}
The product $J_{c,p}u_{c,p}$ approximates the incremental contact work along the cut-advance direction. We define an empirical work threshold as $W_{\mathrm{th},p}=G_c(V_p^0)^{2/3}(1-\zeta)/\zeta$, where $G_c$ is the material cutting-energy scale per unit area and $\zeta\in(0,1)$ parameterizes blade sharpness. Within the front region, the positive component of the contact impulse along $\mathbf{d}_k$ is applied until the accumulated directional work reaches this threshold, i.e., $D_p=1$. This impulse component is then released. The variable $D_p$ affects only this contact component and does not modify the constitutive stress.

Side assignment is determined by the swept blade geometry. An unassigned particle within the cutting band receives side digit when it lies behind the blade edge ($f_{p,k}\le 0$):
\begin{equation}
  d_{p,k}=
  \begin{cases}
    1,&s_{p,k}<0,\\
    2,&s_{p,k}\ge0.
  \end{cases}
  \label{eq:method-side-assignment}
\end{equation}
Assigning a side label does not instantaneously change the particle's position, velocity, or stress. Instead, the assignment changes P2G compatibility and contact selection beginning with the next TLMPM substep, as described in Section~\ref{sec:method-topology}.

\subsection{Topology-Aware Grid Transfer}
\label{sec:method-topology}

Cut-side labels determine how particles share grid states. For $K$ cuts, each particle $p$ stores one ternary side digit $d_{p,k}\in\{0,1,2\}$ per cut: 0 indicates that the particle has not yet been reached by cut $k$, while 1 and 2 indicate the two opposite sides of that cut. The cut-side code of particle $p$ is then defined as
\begin{equation}
  c_p=\sum_{k=0}^{K-1}d_{p,k}3^k.
  \label{eq:method-ternary-code}
\end{equation}
For each cut, a digit can change only from 0 to 1 or 2. Existing nonzero digits persist as subsequent cuts assign values to additional digits, giving each particle a unique code even at cut intersections. This representation supports progressive and intersecting cuts without particle duplication.

Two codes conflict if they encode opposite sides of at least one cut:
\begin{equation*}
\operatorname{conflict}(a,b)
\Longleftrightarrow
\exists k:
(d_k(a),d_k(b))\in\{(1,2),(2,1)\}.
\end{equation*}
Here, $d_k(\cdot)$ extracts the $k$-th ternary digit of a code. The codes are compatible otherwise. In particular, an unassigned digit of 0 is compatible with both sides. Thus, intact material ahead of the blade remains coupled, while the cut faces behind it remain separated (Fig.~\ref{fig:method-grid-coupling}(a)).

Let $\mathcal{S}_i^c=(m_i^c,\mathbf{p}_i^{n,c},\Delta\mathbf{p}_i^c)$ denote the independently accumulated state of logical grid $c$ at node $i$ after P2G transfer and impulse scattering. Here, $\Delta\mathbf{p}_i^c$ contains the internal, external, and contact impulses. With $M_i=\sum_c m_i^c$, we couple compatible nodal states through antisymmetric pairwise exchanges that conserve momentum:
\begin{equation}
  \bar{\mathcal{S}}_i^a
  =\mathcal{S}_i^a+\frac{1}{M_i}
  \sum_{\substack{b\ne a\\
  \neg\operatorname{conflict}(a,b)}}
  \left(m_i^a\mathcal{S}_i^b-m_i^b\mathcal{S}_i^a\right).
  \label{eq:method-compatible-coupling}
\end{equation}
Each logical grid retains its P2G mass, $\bar{m}_i^a=m_i^a$. The momentum coupling preserves any common velocity shared by compatible occupied grids at a node because each pairwise momentum exchange then vanishes. The coupled momentum and accumulated impulse yield the updated nodal velocity:
\begin{equation}
  \mathbf{v}_i^{n+1,a}
  =\frac{\bar{\mathbf{p}}_i^{n,a}
  +\overline{\Delta\mathbf{p}}_i^a}{m_i^a},
  \qquad m_i^a>0.
  \label{eq:method-coupled-grid-update}
\end{equation}
G2P reads from each particle's own logical grid $c_p$. Compatible grids exchange momentum, whereas conflicting grids remain separate even when their reference stencils overlap, allowing opposite cut sides to evolve independently. Our explicit contact model nevertheless allows the opposing faces to recontact without reconnecting (Section~\ref{sec:method-contact}).
A narrow cutting band may leave unassigned particles that indirectly couple opposing cut sides. In practice, we use $h_k=6\Delta x$, where $\Delta x$ is the reference-grid spacing.

We apply the same compatibility coupling during the initial P2G transfer and the MUSL re-scattering of updated particle momentum, ensuring that both operations follow the same material connectivity. The appendix proves momentum conservation for PIC/FLIP transfers and MUSL re-scattering.

\subsection{Contact Model}
\label{sec:method-contact}

Separate logical grids give particles on opposite cut sides independent velocities but cannot prevent interpenetration in the current configuration. We therefore detect material-point contact and scatter the resulting impulses through fixed reference stencils. Following the TLMPM contact model~\cite{deVaucorbeil2021Contact}, we use a unified formulation for inter-object contact, self-contact, cut-face recontact, and rigid-tool contact. All interactions use the same Coulomb friction and impulse-scattering procedure, with persistent contacts retaining tangential history for static friction.

\paragraph{Material--material contact}
Particle $p$ is represented for contact by a sphere of radius $R_p=s_R(V_p^0)^{1/3}/2$, where $s_R$ is the contact-radius scale.  For a pair $(p,q)$, let $\mathbf{n}_{pq}=(\mathbf{x}_p-\mathbf{x}_q)/ \|\mathbf{x}_p-\mathbf{x}_q\|$ and $v_{n,pq}=(\mathbf{v}_p-\mathbf{v}_q)\cdot\mathbf{n}_{pq}$.  We use the predicted overlap
\begin{equation}
  \widehat{\delta}_{pq}
  =R_p+R_q-
  \left[\|\mathbf{x}_p-\mathbf{x}_q\|
  +\min(\Delta t\,v_{n,pq},0)\right]
  \label{eq:method-pair-overlap}
\end{equation}
and the reduced mass $m_{pq}^{\mathrm{eff}}=m_pm_q/(m_p+m_q)$ to obtain the normal force
\begin{equation}
  f_{n,pq}
  =s_n\frac{m_{pq}^{\mathrm{eff}}}{\Delta t^2}
  \max(\widehat{\delta}_{pq},0).
  \label{eq:method-pair-normal}
\end{equation}
Here, $s_n$ scales the normal contact response. For self-contact, nearby reference neighbors are excluded because they are already coupled by the continuum discretization. Pairs with conflicting cut-side codes are exempt from this exclusion and remain candidates for crack-face contact.

To model static friction, we adapt the tangential spring history from TinyDEM~\cite{Vetter2026TinyDEM}.  Each persistent contact stores a tangential spring displacement $\mathbf{u}_t$, which is transported to the current tangent plane at each step to obtain $\widetilde{\mathbf{u}}_t$.  For the material pair defined above, let $\mathbf{v}_{\mathrm{rel}}=\mathbf{v}_p-\mathbf{v}_q$, $\mathbf{n}=\mathbf{n}_{pq}$, $f_n=f_{n,pq}$, and $m_{\mathrm{eff}}=m_{pq}^{\mathrm{eff}}$.  We compute the tangential relative velocity $\mathbf{v}_t$ and trial tangential force $\mathbf{f}_t^{\mathrm{tr}}$:
\begin{align}
  \mathbf{v}_t
  &=\mathbf{v}_{\mathrm{rel}}
  -(\mathbf{v}_{\mathrm{rel}}\cdot\mathbf{n})\mathbf{n},\nonumber\\
  \mathbf{f}_t^{\mathrm{tr}}
  &=-k_t\widetilde{\mathbf{u}}_t-c_t\mathbf{v}_t,\nonumber\\
  k_t&=\gamma_k\frac{m_{\mathrm{eff}}}{\Delta t^2},
  \qquad
  c_t=2\gamma_c\sqrt{m_{\mathrm{eff}}k_t}.
  \label{eq:method-friction-trial}
\end{align}
Here, $k_t$ and $c_t$ are the tangential stiffness and damping coefficients. $\gamma_k$ scales the stiffness relative to $m_{\mathrm{eff}}/\Delta t^2$, while $\gamma_c$ sets the damping ratio relative to critical damping.  The static and dynamic friction coefficients are $\mu_s$ and $\mu_d$. The trial force is retained while $\|\mathbf{f}_t^{\mathrm{tr}}\|\le\mu_s f_n$.  During sticking, we store $\widetilde{\mathbf{u}}_t+\Delta t\,\mathbf{v}_t$ as the tangential history for the next step.  Once the static limit is exceeded, the contact switches to kinetic friction at the dynamic Coulomb bound $\mu_d f_n$, and the stored displacement is adjusted to match the applied force. The same history-based model is used for material--material and material--rigid contact.

The resulting pairwise impulse $\mathbf{J}_{pq}$ is applied as equal-and-opposite impulses to $p$ and $q$ and then scattered to their respective logical grids:
\begin{equation}
  \Delta\mathbf{p}_i^{c_p}\mathrel{+}=N_{ip}\mathbf{J}_{pq},
  \qquad
  \Delta\mathbf{p}_j^{c_q}\mathrel{-}=N_{jq}\mathbf{J}_{pq}.
  \label{eq:method-pair-scatter}
\end{equation}
Indices $i$ and $j$ range over the stencils of particles $p$ and $q$, respectively. Thus, contact is detected in the current configuration, while the resulting impulse is incorporated into the fixed-reference TLMPM update (Fig.~\ref{fig:method-grid-coupling}(b)).

\paragraph{Tool--material contact}
For particle $p$, we use the spherical contact proxy defined above. A closest-point query on the tool collision geometry returns the surface point $\mathbf{x}_c$, outward normal $\mathbf{n}_B$, and signed sphere--surface gap $\phi_p$. With the tool pose and twist defined above, the relative contact velocity and predicted gap are
\begin{align}
  \mathbf{v}_{\mathrm{rel}}
  &=\mathbf{v}_p-
  \left[\mathbf{v}_B+
  \boldsymbol{\omega}_B\times(\mathbf{x}_c-\mathbf{x}_B)\right],\nonumber\\
  \widehat{\phi}_p
  &=\phi_p+\Delta t\min(
  \mathbf{v}_{\mathrm{rel}}\cdot\mathbf{n}_B,0).
  \label{eq:method-rigid-contact-state}
\end{align}
As in material--material contact, we use the particle mass $m_p$ for the spherical proxy.  Let $w_B$ denote the rigid body's inverse effective mass at $\mathbf{x}_c$ along $\mathbf{n}_B$, with $w_B=0$ for a kinematically prescribed tool.  The combined contact effective mass is $m_{pB}^{\mathrm{eff}}=(m_p^{-1}+w_B)^{-1}$.
Using $m_{pB}^{\mathrm{eff}}$, the one-step normal response corresponding to Eq.~\eqref{eq:method-pair-normal} gives
\begin{equation}
  J_{n,pB}
  =s_n\frac{m_{pB}^{\mathrm{eff}}}{\Delta t}
  \max(-\widehat{\phi}_p,0).
  \label{eq:method-rigid-normal-impulse}
\end{equation}
Here, $\widehat{\phi}_p<0$ indicates predicted particle--tool penetration. The tangential component uses the same history-based Coulomb response, with Eq.~\eqref{eq:method-friction-trial} evaluated using $m_{\mathrm{eff}}=m_{pB}^{\mathrm{eff}}$ and $f_n=J_{n,pB}/\Delta t$. We scatter the final impulse $\mathbf{J}_{pB}$ to the particle's logical grid and return the equal-and-opposite reaction force $\mathbf{F}_{B,p}$ and torque $\boldsymbol{\tau}_{B,p}$ to the tool interface:
\begin{align}
  \Delta\mathbf{p}_i^{c_p}
  &\mathrel{+}=N_{ip}\mathbf{J}_{pB},\nonumber\\
  \mathbf{F}_{B,p}
  &=-\frac{\mathbf{J}_{pB}}{\Delta t},
  \qquad
  \boldsymbol{\tau}_{B,p}
  =(\mathbf{x}_c-\mathbf{x}_B)\times\mathbf{F}_{B,p}.
  \label{eq:method-rigid-reaction}
\end{align}
Thus, a single contact impulse updates the material motion, generates the corresponding tool reaction, and provides the contact work used by the cutting-resistance model.

\subsection{Reference-Grid Surface Reconstruction}
\label{sec:method-surface}

For visualization, we reconstruct the visible boundary on a uniform reference grid using the linked-volume representation~\cite{Wu2011BoundarySurfaces} and local dual contouring~\cite{Ju2002DualContouring}. Intersections between primal-grid links and the original boundary or blade-generated cut surfaces are stored as Hermite data, each comprising a point and a normal. Sweeping the finite blade edge disconnects the intersected links. Within each incident cell, connected components of the remaining links define local material regions. We update only affected cells and duplicate surface representatives across distinct regions to form the two coincident faces of a cut.

Surface motion is driven by TLMPM through one-way coupling. Each surface representative inherits the cut-side code of its local material region and is embedded using fixed trilinear weights in the corresponding compatible TLMPM motion field. It therefore follows the deformation of its side without introducing additional dynamics. Link connectivity determines the visual topology, while the side-aware embedding keeps the surface motion consistent with the mechanical discontinuity.

\subsection{GPU Implementation and Rigid-Tool Interface}
\label{sec:method-implementation}
We implement stress evaluation, TLMPM transfers, contact handling, logical-grid aggregation, and cutting-state updates as GPU kernels. Active topology codes are compacted into contiguous physical grid slots, so memory and computation scale with the number of active states rather than the full $3^K$ code space. The external rigid-body system may update the tool at a lower frequency than TLMPM\@. The tool pose and twist remain fixed during the intervening TLMPM substeps, while reaction impulses are accumulated and returned as an interval-averaged wrench that preserves the net linear and angular impulses.

\section{Experiments}
\label{sec:experiments}

\subsection{Comparison with Existing Cutting Mechanisms}

Unlike the mesh-based DiSECt method~\cite{Heiden2021DiSECt}, \textsc{BladeMaster} does not require cut surfaces to be predefined. To evaluate this online cut-surface generation capability, we release a \(0.1\,\mathrm{kg}\) knife above a soft elastic block with an initial tilt of \(25^\circ\). The block has Young's modulus \(E=6\,\mathrm{kPa}\), Poisson's ratio \(\nu=0.4\), and density \(\rho=1100\,\mathrm{kg\,m^{-3}}\) (Fig.~\ref{fig:gravity-driven-cutting}(a)--(c)). The knife subsequently moves only under gravity and material reaction forces, which rotate it and deflect its trajectory. Our method generates the cut surface online along this evolving trajectory, allowing the material response to influence the cut geometry. The newly created surfaces later recontact and slide relative to one another without reconnecting. In a separate run with a lower knife-sharpness setting, the knife does not cut through the block under gravity (Fig.~\ref{fig:gravity-driven-cutting}(d)).

\begin{figure}[hbt]
  \centering
  \includegraphics[width=\columnwidth]{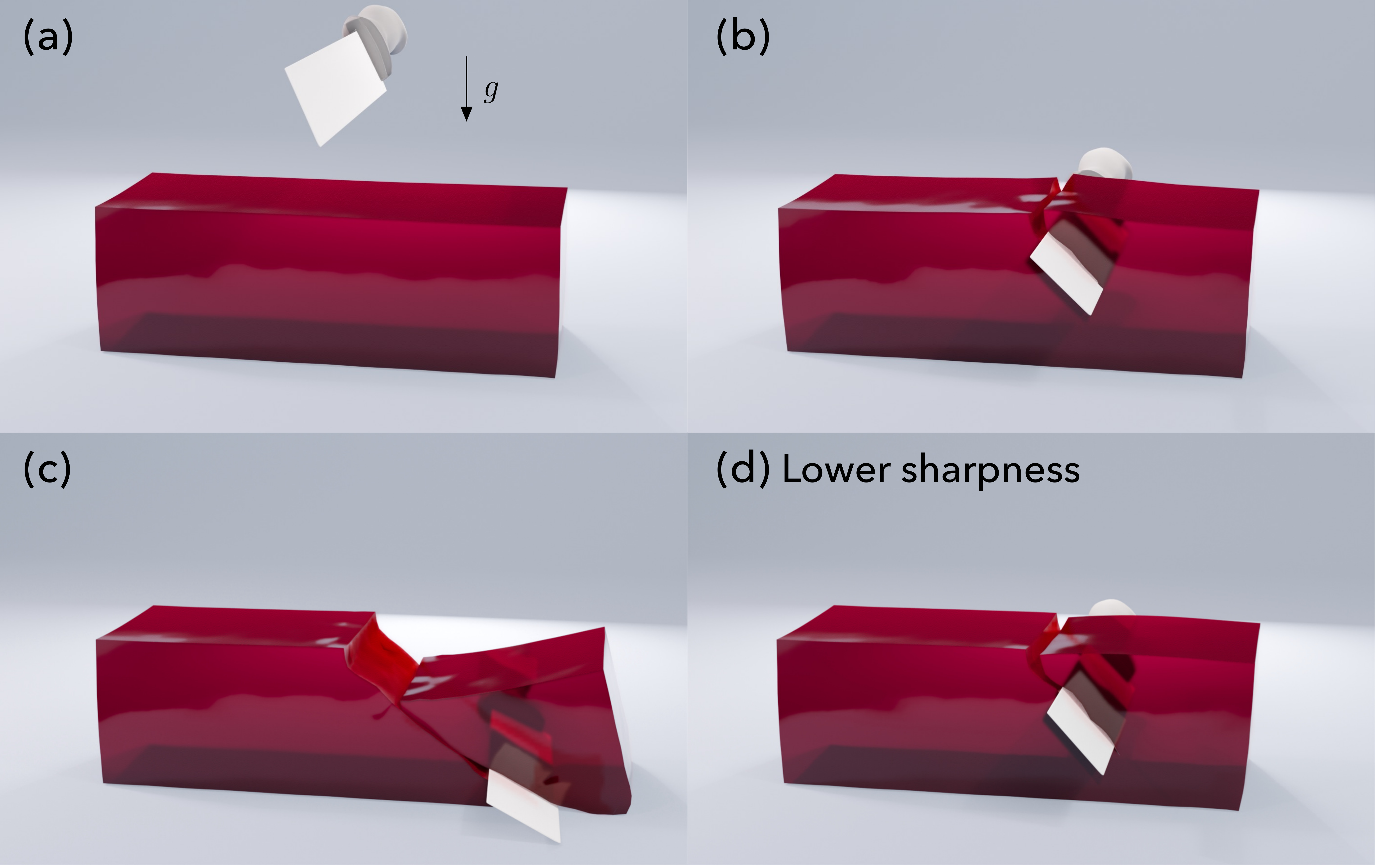}
  \caption{Online cutting with a freely moving knife. (a)--(c) The knife falls under gravity, while material reaction forces rotate it and deflect its cutting trajectory. (d) With a lower knife-sharpness setting, the knife does not cut through the block.}
  \label{fig:gravity-driven-cutting}
\end{figure}

In single-field MPM, opposing cut faces can recouple through shared grid nodes after knife withdrawal. To isolate the effect of cut-side compatibility, we compare \textsc{BladeMaster} with a single-field baseline in which particle--grid transfers ignore cut-side codes once the knife is withdrawn. Both simulations use identical particle sampling and knife motion for a cylindrical body with Young's modulus \(E=10\,\mathrm{kPa}\), Poisson's ratio \(\nu=0.4\), and density \(\rho=1100\,\mathrm{kg\,m^{-3}}\). A simulated Franka arm performs two radial cuts to form a slice, withdraws the knife after each cut, and then inserts it beneath the slice to lift it (Fig.~\ref{fig:persistent-cake-cut}). \textsc{BladeMaster} lifts only the slice, whereas shared-grid coupling in the baseline causes the surrounding material to rise with it. This contrast demonstrates that persistent side labels preserve mechanical separation after knife withdrawal.

\begin{figure}[hbt]
  \centering
  \includegraphics[width=\columnwidth]{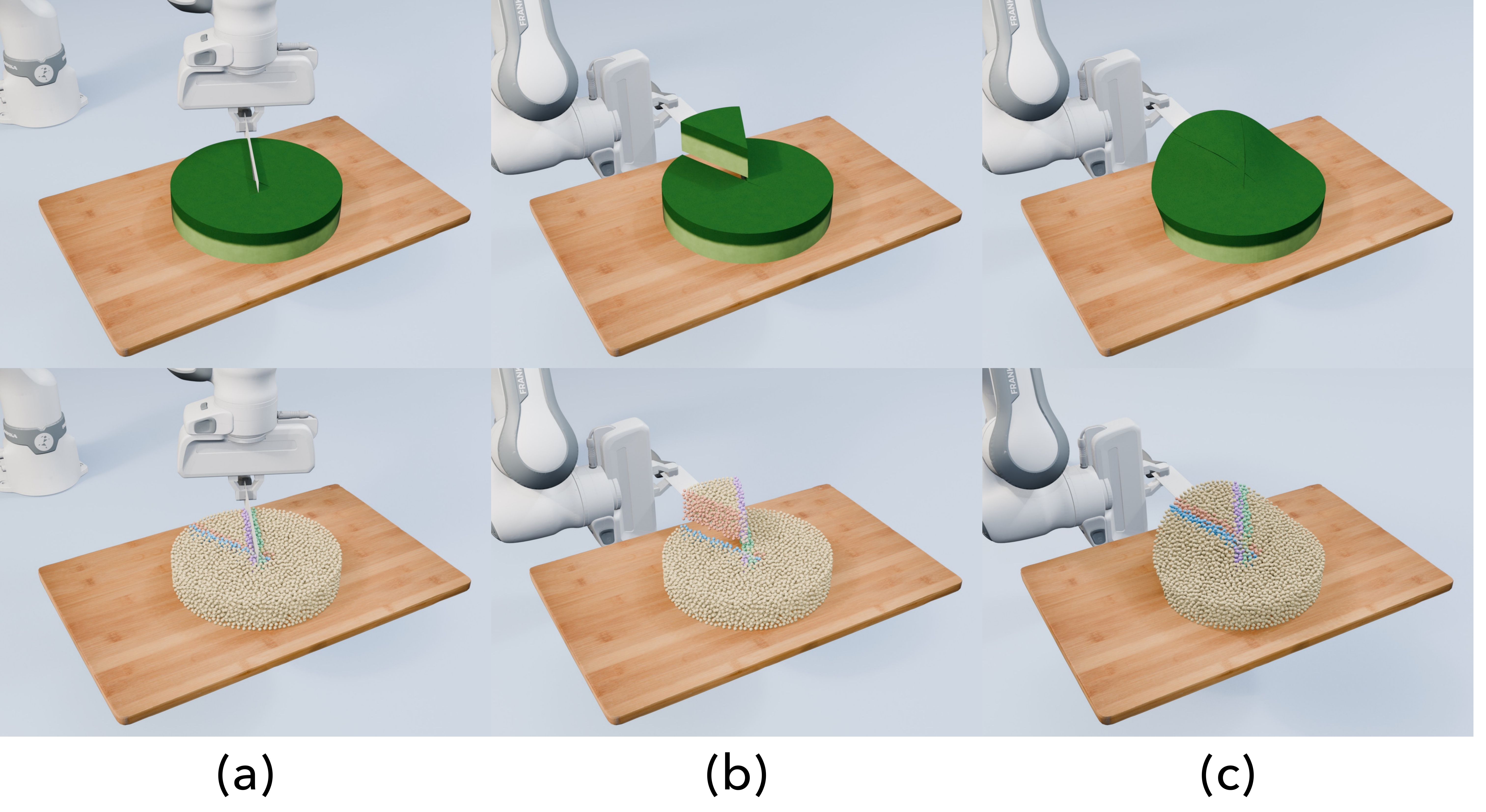}
  \vspace{-6mm}
  \caption{Mechanical separation during slice lifting. Rows show reconstructed surfaces (top) and particles (bottom). (a) Two radial cuts form a slice. After the knife is withdrawn from the cut region, (b) \textsc{BladeMaster} lifts only the slice, whereas (c) the single-field baseline drags the surrounding material upward.}
  \label{fig:persistent-cake-cut}
\end{figure}

Figure~\ref{fig:cressim-comparison} compares how \textsc{BladeMaster} and CRESSim-MPM~\cite{Ou2025CRESSim} preserve cut history. Two parallel cuts divide a rectangular block into three pieces. Both methods use identical particle sampling, elastic parameters, blade trajectories, and a \(0.5\,\mathrm{mm}\)-thick blade.
Halfway through the first cut, the mean gap width is \(0.654\,\mathrm{mm}\) with \textsc{BladeMaster}, close to the blade thickness, whereas CRESSim-MPM produces a mean gap width of \(4.059\,\mathrm{mm}\) (Fig.~\ref{fig:cressim-comparison}(a1) and (a2)). After both cuts, the volume contraction of the middle piece, estimated from tetrahedral volumes, is \(0.506\%\) with \textsc{BladeMaster} and \(8.746\%\) with CRESSim-MPM (Fig.~\ref{fig:cressim-comparison}(b1) and (b2)).
After blade withdrawal, we assign the same initial leftward velocity to only the left piece in each simulation. Figure~\ref{fig:cressim-comparison}(c1) and (c2) shows the resulting states after the same elapsed time. With \textsc{BladeMaster}, the left piece moves away while the middle and right pieces remain near their original positions. With CRESSim-MPM, shared nodal velocities redistribute momentum among the pieces during subsequent transfers, causing the middle and right pieces to move leftward and the left piece to travel a shorter distance. These results demonstrate that \textsc{BladeMaster} supports independent manipulation after successive cuts by retaining the accumulated changes in material connectivity during subsequent motion.

\begin{figure}[hbt]
  \centering
  \includegraphics[width=\columnwidth]{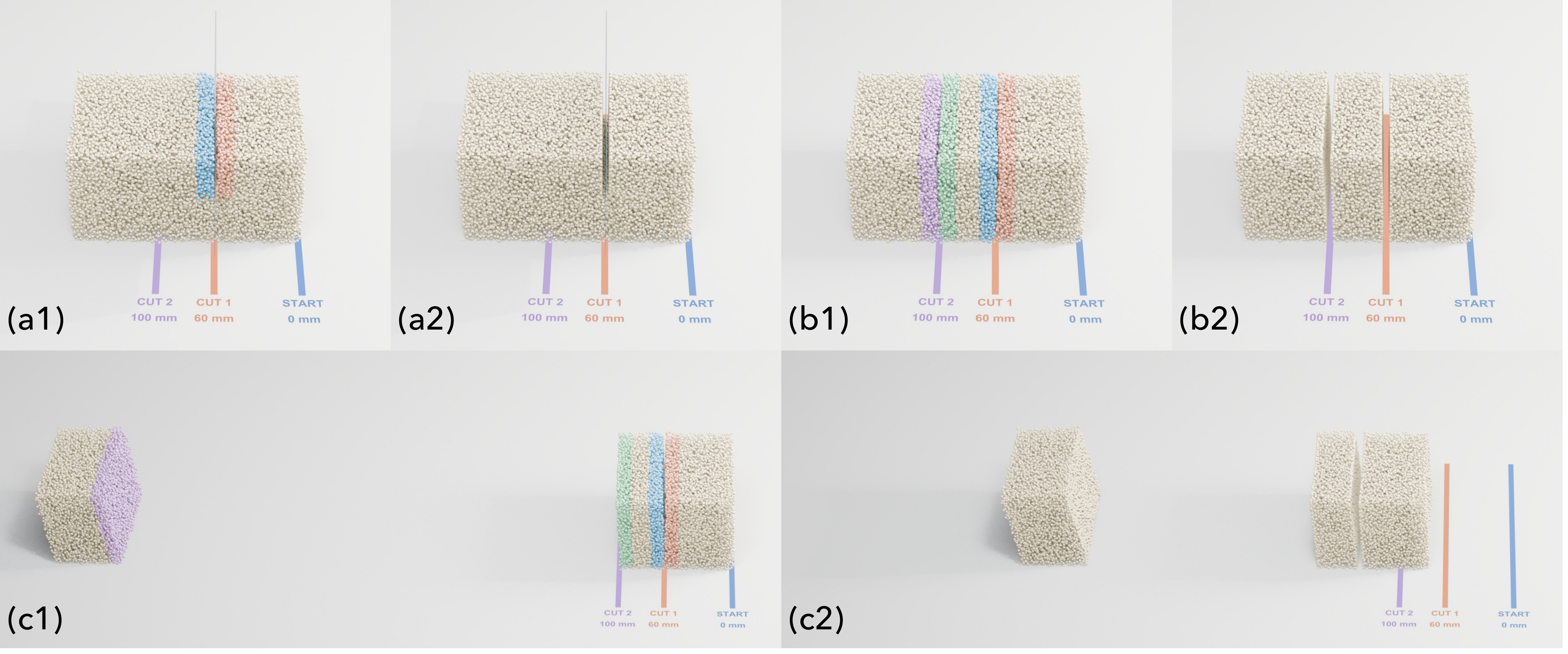}
  \caption{Sequential cutting and manipulation: \textsc{BladeMaster} (a1)--(c1) and CRESSim-MPM (a2)--(c2). (a) First cut at half depth. (b) Both cuts completed. (c) States at the same elapsed time after giving only the left piece an initial leftward velocity. Colored bands identify cut sides for \textsc{BladeMaster}, and ground markers provide fixed position references.}
  \label{fig:cressim-comparison}
  \vspace{-0.6cm}
\end{figure}

\subsection{Cutting Scenarios}

The following scenarios illustrate the range of cutting interactions supported by our method. Figure~\ref{fig:throwing-knife} shows three freely moving knives launched simultaneously toward the same deformable block. Their initial speeds are \(0.5\), \(0.8\), and \(1.1\,\mathrm{m\,s^{-1}}\), ordered from left to right. Their trajectories are not prescribed. The knives and deformable body evolve through two-way coupling: material reactions decelerate the knives, while higher initial speeds produce greater penetration depths. Because the cutting-progress updates act on separate particle sets, the three cuts can be simulated concurrently. The block is modeled using fixed-corotated elasticity with von Mises plasticity, with \(E=20\,\mathrm{kPa}\), \(\nu=0.4\), \(\rho=1200\,\mathrm{kg\,m^{-3}}\), and a yield stress of \(2\,\mathrm{kPa}\).

\begin{figure}[hbt]
  \centering
  \includegraphics[width=\columnwidth]{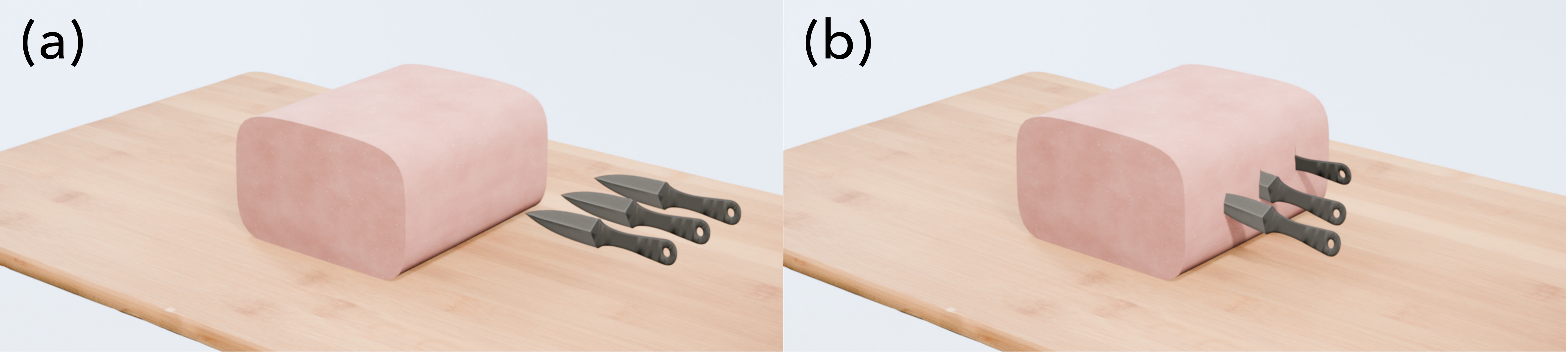}
  \caption{Impact and penetration of multiple knives into a deformable block. (a) Three freely moving knives are launched simultaneously, with initial speeds increasing from left to right. (b) Material reaction forces decelerate the knives, resulting in different penetration depths.}
  \label{fig:throwing-knife}
\end{figure}

In Fig.~\ref{fig:knife-tip-scoring}, a Franka arm produces intersecting partial cuts by guiding the knife tip along a cross-shaped path. To isolate the effect of penetration depth, we repeat the same in-plane motion while varying only how deeply the knife tip penetrates the material. The shallow pass produces light surface scores, whereas the deeper pass creates more pronounced incisions. In both cases, the two cut paths intersect at the center but do not pass through the body. The block is modeled with \(E=6\,\mathrm{kPa}\), \(\nu=0.35\), and \(\rho=1100\,\mathrm{kg\,m^{-3}}\).

\begin{figure}[hbt]
  \centering
  \includegraphics[width=\columnwidth]{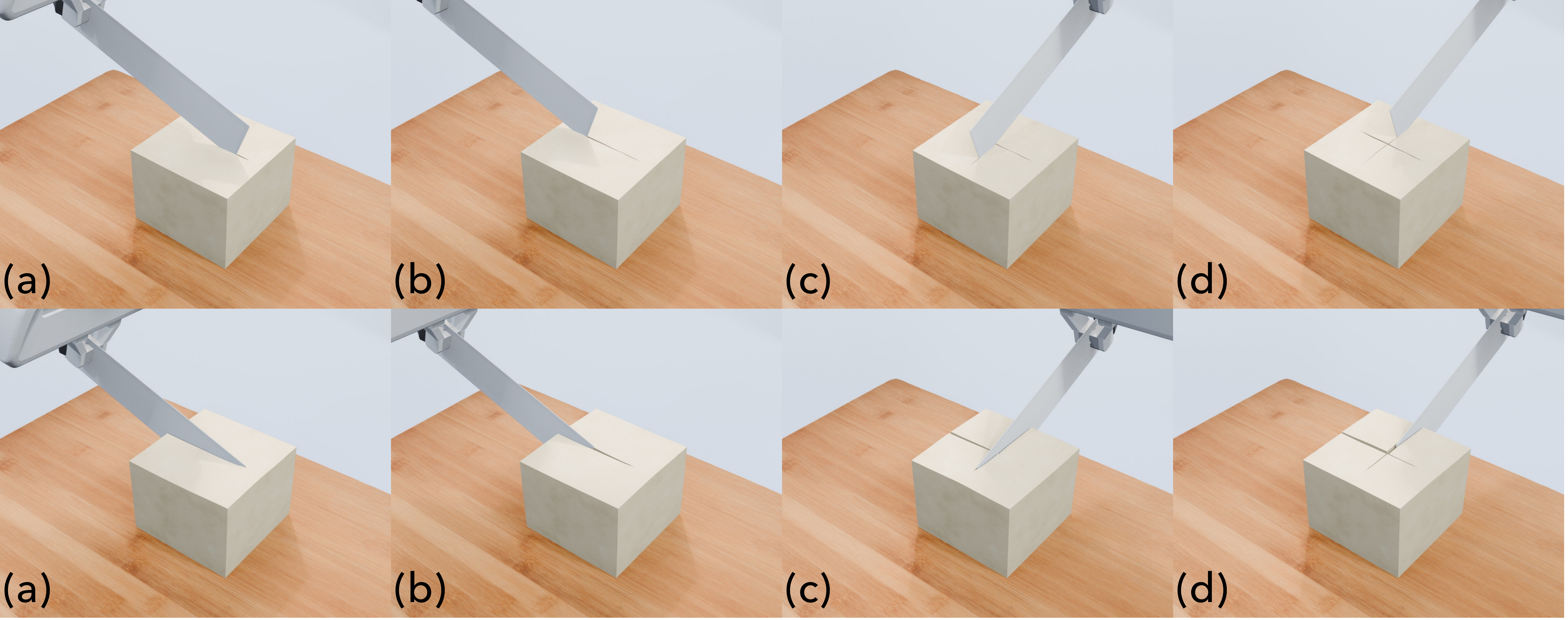}
  \caption{Intersecting partial cuts along a cross-shaped path at two penetration depths. Each row shows the cutting sequence from (a) to (d). The shallow pass (top) leaves faint surface scores, whereas the deeper pass (bottom) produces more pronounced incisions.}
  \label{fig:knife-tip-scoring}
\end{figure}

We further demonstrate sequential slicing with a Franka arm (Fig.~\ref{fig:sequential-slicing}). The block uses the same material parameters as the moving-knife example. During each cut, the opening develops above the blade edge while the uncut material below remains connected. As successive cuts are completed, the released slices topple and contact the board and previously cut slices, forming an overlapping pile. This sequence demonstrates progressive cutting, persistent separation, and contact among the resulting pieces within a single task.

\begin{figure}[hbt]
  \centering
  \includegraphics[width=\columnwidth]{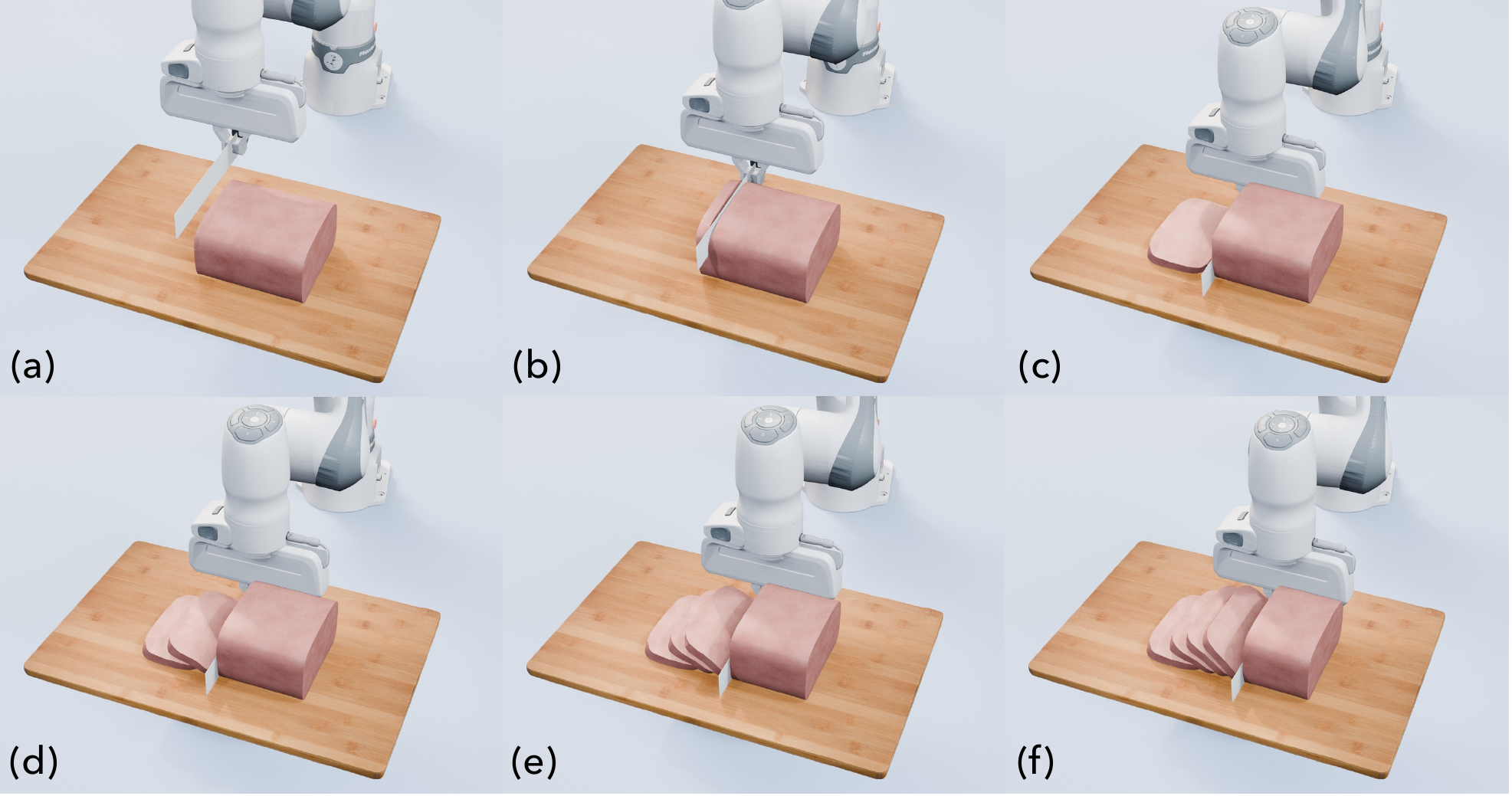}
  \caption{Sequential slicing in simulation. (a) Initial configuration. (b) Halfway through the first cut, the upper portion has separated while the lower portion remains connected. (c)--(f) Successive cuts produce thin slices that topple under gravity and accumulate on the cutting board.}
  \label{fig:sequential-slicing}
\end{figure}

\subsection{Real-world experiment \& Cutting Resistance Evaluation}

We quantitatively compare the cutting resistance computed by our method with estimates from a real robot's joint-torque measurements on a single vertical cut through a banana. We model the banana tissue using Young's modulus $E=0.4\,\mathrm{MPa}$, Poisson's ratio $\nu=0.45$, density $\rho=1100\,\mathrm{kg\,m^{-3}}$, and yield stress $30\,\mathrm{kPa}$. The simulated knife has a mass of $0.1\,\mathrm{kg}$ and a friction coefficient of $0.2$. We set $G_c=700\,\mathrm{J\,m^{-2}}$ and $\zeta=0.9$. We visualize the computed and measured resistance-force profiles (Fig.~\ref{fig:cut-reaction}). We additionally report an averaged trace over 20 simulations with different particle samplings and a refined, higher-resolution simulation. We observe that the simulated and experimental resistance forces are comparable in magnitude during cutting before knife--support contact. The averaged and refined traces exhibit fewer fluctuations and more closely match the experimental measurements.

We qualitatively compare our simulation with a real-world banana-cutting experiment (Fig.~\ref{fig:banana-teaser}).
The same end-effector trajectory is executed in both the simulation and the real-world experiment. After each cut, the knife pushes the resulting pieces aside. Across the illustrated stages, the simulation closely matches the experiment in the separation and displacement of the pieces. This task demonstrates interleaved cutting and manipulation while preserving the separation created by earlier cuts.

\begin{figure}[hbt]
  \centering
  \includegraphics[width=0.95\columnwidth]{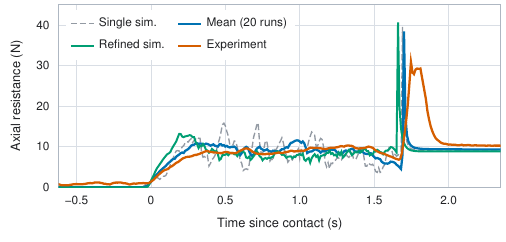}
  \vspace{-4mm}
  \caption{Axial resistance during banana cutting.  Simulation curves show a single run, the pointwise mean over 20 runs with different particle samplings at the same resolution, and a refined run with one-quarter the grid spacing and $64\times$ as many particles.  Total resistance includes material and rigid-support reactions.  The late peaks correspond to knife contact with the underlying rigid support.}
  \label{fig:cut-reaction}
  \vspace{-0.6cm}
\end{figure}

\subsection{Runtime and Scalability}

We evaluate \textsc{BladeMaster} on an NVIDIA GeForce RTX 5090. The reported timings cover the simulation loop and exclude surface reconstruction and rendering. Table~\ref{tab:simulation-statistics} summarizes the discretization and runtime statistics of representative examples. All listed cutting examples run faster than real time. At a \(2\,\mathrm{ms}\) time step, Fig.~\ref{fig:runtime-grid-scaling} shows how runtime scales with the numbers of particles and cuts. Compact allocation stores only occupied topology states, requiring far fewer grid slots than the full \(3^K\) code space.

\begin{table}[hbt]
  \centering
  \makeatletter
  \long\def\@makecaption#1#2{%
    \noindent\parbox{\hsize}{\footnotesize #1: #2}\par
    \vskip 0.5\baselineskip}
  \makeatother
  \caption{\textbf{Simulation statistics and runtime.} Wall/sim.\ denotes the ratio of wall-clock time to simulated time, with values below 1 indicating faster-than-real-time execution.}
  \label{tab:simulation-statistics}
  \footnotesize
  \begin{tabular*}{\columnwidth}{@{\extracolsep{\fill}}lcccc@{}}
    \toprule
    Example & \# Particles & \(\Delta t\) [ms] & \# Cuts & Wall/sim. \\
    \midrule
    Gravity-driven cutting & 8960 & 2 & 1 & 0.33 \\
    Slice lifting & 8960 & 2 & 2 & 0.36 \\
    Moving knives & 7922 & 1 & 3 & 0.62 \\
    Knife-tip scoring & 3584 & 2 & 2 & 0.30 \\
    Sequential slicing & 7922 & 1 & 4 & 0.56 \\
    \bottomrule
  \end{tabular*}
  \vspace{-4mm}
\end{table}

\begin{figure}[hbt]
  \centering
  \begin{minipage}[t]{0.45\columnwidth}
    \centering
    \includegraphics[width=\linewidth]{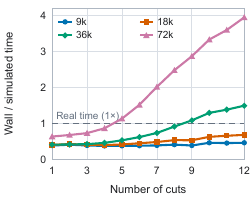}
    \par\smallskip
    \vspace{-2mm}
    {\footnotesize (a)}
  \end{minipage}\hfill
  \begin{minipage}[t]{0.45\columnwidth}
    \centering
    \includegraphics[width=\linewidth]{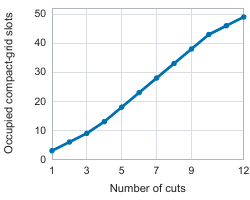}
    \par\smallskip
    \vspace{-2mm}
    {\footnotesize (b)}
  \end{minipage}
  \caption{Runtime and compact-grid usage for parallel cuts.  (a) Ratio of wall-clock time to simulated time for four particle counts.  The dashed line marks real-time execution (\(1\times\)).  (b) Number of occupied compact-grid slots at each cut count.}
  \label{fig:runtime-grid-scaling}
\end{figure}

\section{Conclusion}
\label{sec:conclusion}

We introduced \textsc{BladeMaster}, a GPU-accelerated TLMPM framework that generates progressive cuts from evolving blade motion without predefined cut surfaces. Persistent particle side labels maintain coupling within intact material while preventing direct grid coupling between opposing cut faces after tool withdrawal. Material--material contact allows cut faces to recontact without reconnecting, while two-way tool--material coupling allows material reaction forces to influence tool motion and cut geometry. Simulation experiments demonstrate cutting and subsequent manipulation without artificial grid coupling across completed cuts, with faster-than-real-time performance on representative tasks. A real robot experiment shows qualitatively similar separation and manipulation of the cut pieces.

Several limitations remain. Tool--material contact uses spherical particle proxies and grid-scattered impulses but does not strictly guarantee nonpenetration. The model introduces topological separation only along blade-induced cuts and therefore does not capture stress-driven crack initiation, crack propagation, or failure under non-cutting loads. Compact allocation avoids instantiating all $3^K$ states, but the number of active states may still grow with the number and arrangement of cuts. Despite these limitations, by combining online cut generation, persistent separation, and real-time performance, \textsc{BladeMaster} represents an important step toward efficient policy learning and motion optimization for robotic cutting tasks.

\appendix[Momentum Conservation]
\label{app:linear-momentum}

\begin{proposition}
\label{prop:momentum-conservation}
The coupling in Eq.~\eqref{eq:method-compatible-coupling} preserves momentum in the PIC/FLIP transfer and MUSL re-scatter of~\cite{deVaucorbeil2020TLMPM} for any common blend ratio $\beta\in[0,1]$.  This assumes fixed particle masses $m_p$ and cut-side codes $c_p$ within a substep, with $\sum_iN_{ip}=1$.
\end{proposition}

\begin{proof}
For each compatible pair $(a,b)$, the exchange in Eq.~\eqref{eq:method-compatible-coupling} is the negative of that for $(b,a)$.  Thus, at every occupied node, $\sum_a\bar{\mathcal{S}}_i^a=\sum_a\mathcal{S}_i^a$ and $\bar{m}_i^a=m_i^a$.
Combining G2P interpolation with the PIC/FLIP velocity update gives
\begin{equation}
  \mathbf{v}_p^{n+1}=\beta\mathbf{v}_p^n
  +\sum_i\frac{N_{ip}}{m_i^{c_p}}
  \left[(1-\beta)\bar{\mathbf{p}}_i^{n,c_p}
  +\overline{\Delta\mathbf{p}}_i^{c_p}\right].
  \label{eq:appendix-g2p}
\end{equation}
Since $\frac{\sum_{p:c_p=a}m_pN_{ip}}{m_i^a}=1$, writing $\mathbf{P}^n=\sum_pm_p\mathbf{v}_p^n$, we obtain
\begin{align}
  \mathbf{P}^{n+1}
  &=\beta\mathbf{P}^n+(1-\beta)\sum_{i,a}\bar{\mathbf{p}}_i^{n,a}
  +\sum_{i,a}\overline{\Delta\mathbf{p}}_i^a\nonumber\\
  &=\mathbf{P}^n+\sum_{i,a}\Delta\mathbf{p}_i^a.
  \label{eq:appendix-momentum-balance}
\end{align}
This holds for arbitrary impulses on occupied fields and every common $\beta$.

MUSL re-scatters $\sum_{p:c_p=a}N_{ip}m_p\mathbf{v}_p^{n+1}$ to each grid.  By the same argument, its total is $\mathbf{P}^{n+1}$, which the coupling preserves by the same pairwise cancellation.
\end{proof}

\bibliographystyle{IEEEtran}
\bibliography{references}

\end{document}